\documentclass[letterpaper, 10pt, conference]{IEEEtran}

\IEEEoverridecommandlockouts

\usepackage{amsmath,amssymb,amsfonts}
\usepackage{graphicx}
\usepackage{cite}
\usepackage{algorithmic}
\usepackage{algorithm}
\usepackage{textcomp}
\usepackage{xcolor}
\usepackage{hyperref}
\usepackage{booktabs}
\usepackage{multirow}
\usepackage[flushleft]{threeparttable}
\newtheorem{theorem}{Theorem}
\newtheorem{proof}{Proof}

\title{\vspace*{0.25in}\LARGE \bf
Robust Multi-Agent Target Tracking in Intermittent Communication Environments via Analytical Belief Merging
}

\author{
Mohamed Abdelnaby, Samuel Honor, and Kevin Leahy%
\thanks{The authors are with the Department of Robotics Engineering, Worcester Polytechnic Institute (WPI), Worcester, MA 01609, USA.
        {\tt\small \{mabdelnaby, sjhonor, kleahy\}@wpi.edu. This work was supported by ONR under grant number N00014-25-1-2225.}}%
}

\begin{document}

\maketitle
\thispagestyle{empty}
\pagestyle{empty}

\begin{abstract}
Autonomous multi-agent target tracking in GPS-denied and communication-restricted environments (e.g., underwater exploration, subterranean search and rescue, and adversarial domains) forces agents to operate independently and only exchange information during brief reconnection windows. Because transmitting complete observation and trajectory histories is bandwidth-exhaustive, exchanging probabilistic belief maps serves as a highly efficient proxy that preserves the topology of agent knowledge. While minimizing divergence metrics to merge these decentralized beliefs is conceptually sound, traditional approaches often rely on numerical solvers that introduce critical quantization errors and artificial noise floors. In this paper, we formulate the decentralized belief merging problem as Forward and Reverse Kullback-Leibler (KL) divergence optimizations and derive their exact closed-form analytical solutions. By deploying these derivations, we mathematically eliminate optimization artifacts, achieving perfect mathematical fidelity while reducing the computational complexity of the belief merge to $\mathcal{O}(N|S|)$ scalar operations. Furthermore, we propose a novel spatially-aware visit-weighted KL merging strategy that dynamically weighs agent beliefs based on their physical visitation history. Validated across tens of thousands of distributed simulations, extensive sensitivity analysis demonstrates that our proposed method significantly suppresses sensor noise and outperforms standard analytical means in environments characterized by highly degraded sensors and prolonged communication intervals.
\end{abstract}

\section{Introduction}
The deployment of multi-agent systems in real-world environments often necessitates operation under severe communication constraints. Multi-robot teams have become increasingly vital for high-stakes applications such as post-disaster subterranean search and rescue \cite{murphy2014disaster}, adversarial pursuit-evasion in GPS-denied zones \cite{chung2011search}, and cooperative underwater exploration \cite{akyildiz2005underwater}. In these domains, continuous communication is functionally impossible. Agents must operate independently for extended periods, updating their local probabilistic beliefs based on noisy sensor data. When brief, intermittent communication windows finally occur, agents must rapidly synchronize their accumulated knowledge \cite{thrun2005probabilistic}. However, transmitting an agent's entire history of trajectories and raw observations is bandwidth-exhaustive and computationally prohibitive. Instead, exchanging local probabilistic belief maps serves as an elegant, compressed proxy.

The core challenge in this proxy approach is the accurate reconciliation of conflicting probabilistic maps. To address the specific challenges of decentralized tracking and belief reconciliation, contemporary literature heavily relies on either Decentralized Partially Observable Markov Decision Process (Dec-POMDP) formalisms \cite{oliehoek2016concise, kaelbling1998planning, amato2013decentralized, lauri2020multi} or Deep Reinforcement Learning (DRL) \cite{lowe2017multi, shi2023cooperative}. However, these methods often rely on brittle assumptions. Dec-POMDP solvers typically require complete prior knowledge of the target's transition dynamics and the environment's reward function—an unrealistic assumption for uncooperative or unpredictable targets. Conversely, DRL approaches function as uninterpretable ``black boxes'' that are highly vulnerable to distributional shifts, such as sudden spikes in sensor degradation. Preliminary empirical evaluations in this study indicate that exchanging local probabilistic belief maps serves as an elegant, compressed proxy that preserves the spatial topology of individual agent knowledge without the overhead of raw data transmission.

In contrast, we propose a ``white-box,'' zero-shot analytical approach that makes no assumptions about the target's underlying Markov Decision Process (MDP). We establish that while minimizing the Kullback-Leibler (KL) divergence is mathematically optimal for belief aggregation, the standard computational implementation of this minimization via linear solvers is fundamentally flawed. 

This paper presents three primary contributions:
\begin{enumerate}
    \item We formulate the decentralized belief merging process as discrete Forward and Reverse KL divergence optimization problems, and explicitly derive their closed-form analytical solutions (Arithmetic and Geometric means).
    \item We demonstrate that standard numerical solvers introduce quantization and bounding errors that severely degrade tracking performance compared to our exact analytical derivations.
    \item We introduce a visit-weighted KL merging technique that leverages the spatial history of the agents to filter out "sensor ghosts" (false positives) from degraded sensors. This method demonstrates superior robustness during long communication blackouts without relying on environmental assumptions or brittle learning-based approximations.
\end{enumerate}

The remainder of this paper is organized as follows: Section II formulates the target tracking and belief merging problem. Section III derives the analytical solutions for KL divergence, comparing them against numerical solvers. Section IV introduces our novel visit-weighted KL strategy. Section V details the planning architecture and provides a step-by-step walkthrough of the execution algorithm. Sections VI and VII present our methodological justification and large-scale experimental results, followed by concluding remarks in Section VIII.

\section{Problem Formulation}
We model the environment as a discrete state space $S$ where a team of $N$ agents seeks to locate a moving target. Let $x^*_t \in S$ denote the true state of the target at time $t$, and $x_{i,t} \in S$ denote the position of agent $i$ at time $t$.

\subsection{Sensor Model and Degradation}
At time $t$, each agent $i \in \{1, \dots, N\}$ receives a binary observation $z_{i,t} \in \{0, 1\}$ indicating the presence or absence of the target at its current location $x_{i,t}$. Based on this observation history, the agent maintains a local belief distribution $b_{i,t}(s) = P(x^*_t = s \mid z_{i, 1:t}, x_{i, 1:t})$, representing the probability that the target occupies state $s \in S$. 

The sensors are characterized by an observation model $O(z \mid x^*, x_i)$ defined by a false positive rate $\alpha$ and a false negative rate $\beta$:
\begin{align}
    P(z_{i,t}=1 \mid x^*_t = x_{i,t}) &= 1 - \beta \\
    P(z_{i,t}=1 \mid x^*_t \neq x_{i,t}) &= \alpha
\end{align}
Under highly degraded sensor settings, a high $\alpha$ leads to frequent ghost target sightings, while a high $\beta$ causes agents to overlook the true target. At each time step, an agent updates its local belief via Bayes' rule:
\begin{equation}
    b_{i,t}(s) = \eta \cdot P(z_{i,t} \mid x^*_t = s, x_{i,t}) \cdot b_{i,t-1}(s)
\end{equation}
where $\eta = \left( \sum_{s'} P(z_{i,t} \mid s', x_{i,t}) b_{i,t-1}(s') \right)^{-1}$ is the standard Bayesian normalization constant ensuring the distribution sums to 1.

\subsection{The Belief Merging Problem}
Given intermittent communication constraints, agents must operate independently for $k$ time steps. This tracking process is deeply interleaved and history-dependent. Because agents navigate and observe the environment independently between communication windows, their local beliefs diverge rapidly based on entirely distinct observation sequences and action trajectories. 

Upon reconnection, the team must fuse their local beliefs $\{b_1, b_2, \dots, b_N\}$ into a consensus distribution $b_{merged}(s)$ that accurately reflects the joint accumulated knowledge, without the exhaustive bandwidth overhead of transmitting raw historical data. To achieve this optimal fusion without incurring computational artifacts, we formalize the goal as follows:

\vspace{1.5mm}
\noindent \textbf{Problem Statement:} \textit{Given a team of $N$ agents with local beliefs $b_{i,t}$ formed via independent, history-dependent observation sequences over a communication blackout interval $k$, determine a computationally efficient consensus distribution $Q^*(s)$ (where $Q^* = b_{merged}$) that optimally minimizes the loss of information relative to the joint local beliefs, while remaining robust to high rates of sensor false positives.}
\vspace{1.5mm}

We address this formulation by turning to the analytical derivations of divergence metrics in the following section.

\section{Analytical Solutions vs. Numerical Approximations}
To optimally fuse the diverging local beliefs of the agents, we must find a consensus distribution that minimizes the information lost during the merge. In information theory, the Kullback-Leibler (KL) divergence, denoted as $D_{KL}(P \parallel Q) = \sum_x P(x) \log \frac{P(x)}{Q(x)}$, is the foundational metric for measuring the distance between two probability distributions. Therefore, formulating the decentralized belief merging problem as a KL divergence optimization guarantees an information-theoretic optimal consensus.

However, implementing these optimizations via numerical linear solvers introduces systemic vulnerabilities in highly degraded domains. To resolve this, we present formal theorems proving their exact closed-form analytical solutions.

\subsection{Forward KL and the Arithmetic Mean}
The Forward KL merging objective seeks a candidate consensus distribution $Q$ (which ultimately serves as our optimal $b_{merged}$) that minimizes the weighted sum of divergences from the agents' local beliefs. Let $w_i \in [0, 1]$ represent the normalized scalar weight assigned to agent $i$, where $\sum_{i=1}^N w_i = 1$. The objective is:
\begin{equation} \label{eq:fwd_kl}
    \min_{Q} \sum_{i=1}^{N} w_i D_{KL}(b_i \parallel Q)
\end{equation}
Subject to the probability constraint $\sum_{s \in S} Q(s) = 1$.

\begin{theorem}
The strict mathematical minimum of the Forward KL divergence merging objective formulated in (\ref{eq:fwd_kl}) is exactly the Arithmetic Mean of the local beliefs.
\end{theorem}

\begin{proof}
To minimize the Forward KL objective formulated in (\ref{eq:fwd_kl}) subject to the probability constraint $\sum_{s \in S} Q(s) = 1$, we expand the divergence definition and construct the Lagrangian:
\begin{equation}
    \mathcal{L} = \sum_{i=1}^{N} w_i \sum_{s \in S} b_i(s) \log \frac{b_i(s)}{Q(s)} + \lambda \left( \sum_{s \in S} Q(s) - 1 \right)
\end{equation}
Differentiating with respect to $Q(s)$ and setting the derivative to zero yields:
\begin{equation}
    \frac{\partial \mathcal{L}}{\partial Q(s)} = -\frac{1}{Q(s)} \sum_{i=1}^{N} w_i b_i(s) + \lambda = 0
\end{equation}
Rearranging to solve for $Q(s)$, we obtain:
\begin{equation}
    \lambda Q(s) = \sum_{i=1}^{N} w_i b_i(s)
\end{equation}
To find the Lagrange multiplier $\lambda$, we sum both sides over all states $s \in S$:
\begin{equation}
    \lambda \sum_{s \in S} Q(s) = \sum_{i=1}^{N} w_i \sum_{s \in S} b_i(s)
\end{equation}
Because $\sum_{s \in S} Q(s) = 1$, $\sum_{s \in S} b_i(s) = 1$, and the weights sum to unity ($\sum_{i=1}^N w_i = 1$), the equation simplifies to $\lambda (1) = 1$. Substituting $\lambda = 1$ back into the derivative solution provides the exact closed-form analytical minimum:
\begin{equation}
    Q(s) = \sum_{i=1}^{N} w_i b_i(s)
\end{equation}
This confirms the exact analytical solution is the Arithmetic Mean. $\blacksquare$
\end{proof}

\subsection{Reverse KL and the Logarithmic Opinion Pool}
While Forward KL is ``zero-avoiding'' (penalizing $Q$ for being zero where any $b_i$ is non-zero), the Reverse KL objective is ``zero-forcing,'' prioritizing a consensus that only holds mass where the agents strictly agree. We formulate the Reverse KL objective by minimizing the divergence from the consensus $Q$ to the individual agents:
\begin{equation} \label{eq:rev_kl}
    \min_{Q} \sum_{i=1}^{N} w_i D_{KL}(Q \parallel b_i)
\end{equation}

\begin{theorem}
The strict mathematical minimum of the Reverse KL divergence merging objective formulated in (\ref{eq:rev_kl}) is the Logarithmic Opinion Pool (a normalized geometric mean).
\end{theorem}

\begin{proof}
To minimize (\ref{eq:rev_kl}) subject to $\sum_{s} Q(s) = 1$, we construct the Lagrangian:
\begin{equation}
    \mathcal{L} = \sum_{i=1}^{N} w_i \sum_{s \in S} Q(s) \log \frac{Q(s)}{b_i(s)} + \lambda \left( 1 - \sum_{s \in S} Q(s) \right)
\end{equation}
Differentiating with respect to $Q(s)$ and setting the derivative to zero yields:
\begin{equation}
    \frac{\partial \mathcal{L}}{\partial Q(s)} = \sum_{i=1}^{N} w_i \left( \log \frac{Q(s)}{b_i(s)} + 1 \right) - \lambda = 0
\end{equation}
Solving for $Q(s)$ isolating the $\lambda$ normalization terms provides the exact closed-form analytical minimum:
\begin{equation}
    Q(s) = \frac{\prod_{i=1}^{N} b_i(s)^{w_i}}{\sum_{s' \in S} \prod_{i=1}^{N} b_i(s')^{w_i}}
\end{equation}
This exact analytical solution is recognized in statistical literature as the Logarithmic Opinion Pool \cite{genest1986combining}. $\blacksquare$
\end{proof}

\subsection{The Numerical Quantization Penalty} 
While we have established the exact analytical solutions for KL divergence merging, constrained optimization problems in robotics are frequently solved via numerical approximation (e.g., Mixed-Integer Nonlinear Programming solvers). If one attempts to approximate these information-theoretic metrics numerically, our large-scale evaluations (Fig. \ref{fig:methods}) empirically demonstrate that numerical solvers severely degrade the mathematical fidelity of the belief merge as the grid area scales, we identify two primary mathematical sources for this systemic failure:

\begin{figure*}[thpb]
  \centering
  \IfFileExists{grid_size_effect_2d_log.pdf}{
    \includegraphics[width=0.95\textwidth]{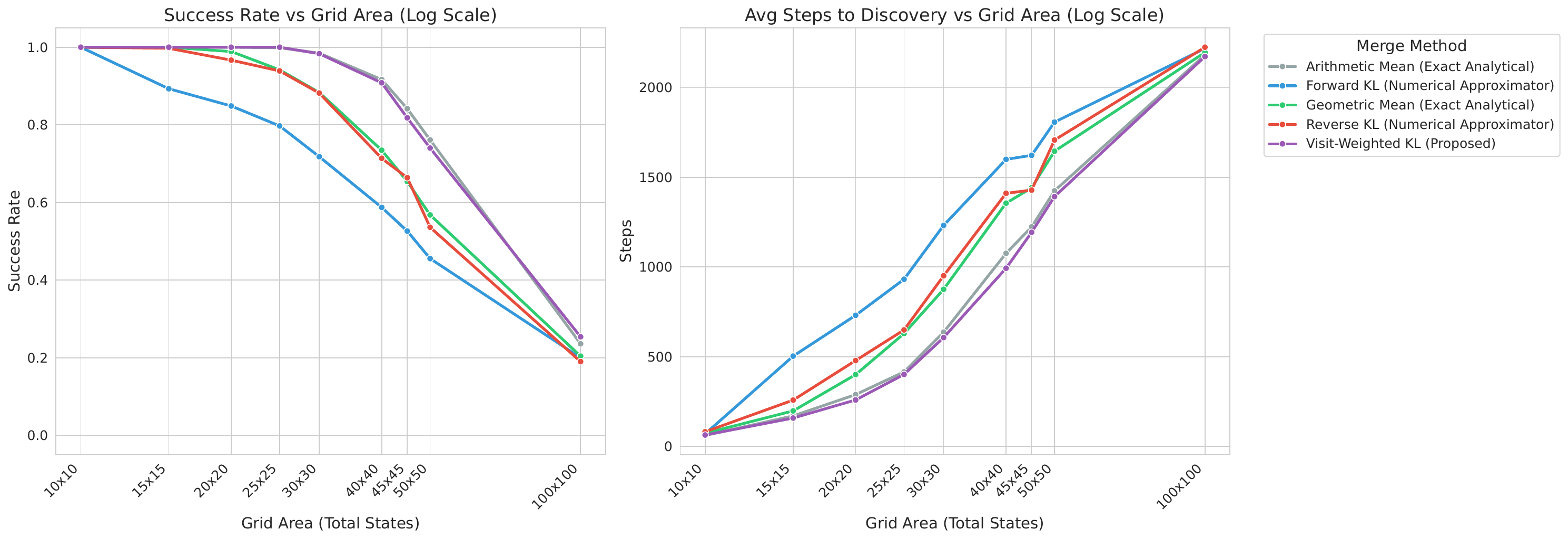}
  }{
    \framebox{\parbox{0.9\textwidth}{\centering\vspace{2cm}\textbf{Image Placeholder: grid\_size\_effect\_2d.pdf}\\\small (Ensure file is uploaded to Overleaf root directory)\vspace{2cm}}}
  }
  \caption{Effect of using linear solvers for our formulated KL and Reverse KL optimizations compared to their exact analytical derivations. Numerical approximations induce high error rates as the grid area scales.}
  \label{fig:methods}
\end{figure*}

\begin{itemize}
    \item \textbf{Artificial Noise Floors:} To prevent undefined logarithmic functions ($\log 0$), standard solvers enforce a feasibility tolerance lower bound (typically $\epsilon \approx 10^{-5}$). In large-scale state spaces, this creates an artificial, uniform noise floor. For a grid of size $|S| = 10,000$, this bounding artifact consumes $|S| \times \epsilon = 10^4 \times 10^{-5} = 0.10$ of the distribution. Consequently, a 10\% of the consensus probability mass is consumed by an artificial mathematical ``fog.'' A faint but true measurement yielding a posterior of $10^{-6}$ is entirely eclipsed by this noise floor.
    \item \textbf{Piecewise Quantization \& Gradient Limits:} Solvers approximate the convex cross-entropy objective using piecewise linear segments. This relies on the function $f(x) = x \log x$, which has the gradient $\nabla_x f(x) = \log x + 1$. As the probability mass approaches zero, the gradient approaches infinity: $\lim_{x \to 0^+} (\log x + 1) = -\infty$. Because piecewise linear sampling cannot accurately capture this infinite vertical slope near zero, the solver induces severe quantization errors for low-probability states.
\end{itemize}

Crucially, this quantization penalty critically fractures the active tracking pipeline. The standard Bayesian update is strictly multiplicative: $b_{i,t}(s) \propto P(z_{i,t} \mid s, x_{i,t}) b_{i,t-1}(s)$. If solver approximations or tolerance thresholds round a small but valid belief to exactly zero, the multiplicative nature of Bayes' rule ensures that $b_{i,t}(s) = 0$ for all future timesteps. Even if the agent later navigates directly to the target's location and receives a definitive positive observation ($P(z_{i,t}=1 \mid s) \approx 1$), the belief remains zero. The true state is permanently excluded from the search space, causing unrecoverable tracking failure. 

By entirely replacing numerical solvers with our proven exact analytical solutions (Theorems 1 and 2), we circumvent these mathematical vulnerabilities. We achieve optimal tracking accuracy while simultaneously reducing the computational overhead of the belief merge from polynomial-time matrix optimizations to strictly $\mathcal{O}(N|S|)$ scalar operations. 

However, while these derived analytical solutions (the Arithmetic Mean and the Logarithmic Opinion Pool) perfectly minimize their respective KL divergences, they remain inherently vulnerable to severe sensor noise. A single highly confident false positive from a degraded sensor can heavily skew the resulting mathematical mean. This vulnerability motivates the spatially-aware approach proposed in Section IV.

\section{visit-weighted KL}
While the Arithmetic Mean optimally solves the forward KL objective, it is highly susceptible to high rates of sensor ghosts (high $\alpha$). If one agent erroneously detects a target in an unvisited region, the arithmetic merge uniformly infects the consensus belief with this false positive.

To counter this, we propose our primary contribution: the visit-weighted KL. Instead of assigning static scalar weights $w_i$ to each agent, we dynamically calculate state-specific spatial weights based on physical visitation history. Let $v_i(s)$ be the number of times agent $i$ has physically occupied state $s$. The spatial weight $w_i(s)$ is defined as:
\begin{equation}
    w_i(s) = \frac{v_i(s) + 1}{\sum_{j=1}^{N} (v_j(s) + 1)}
\end{equation}
The consensus belief is then calculated state-by-state:
\begin{equation}
    b_{merged}(s) = \eta \sum_{i=1}^{N} w_i(s) b_i(s)
\end{equation}
where $\eta$ is a normalization constant. Notice the $+1$ additive (Laplace) smoothing applied to both the numerator and denominator of the spatial weight equation. This non-obvious regularization step is mathematically critical; it prevents zero-division errors in completely unvisited states and ensures that a baseline uniform weighting is preserved for regions of the map that no agent has yet explored. 

Ultimately, this weighting mechanism dictates that if an agent reports a high probability of a target in a state it has actively patrolled, its belief is trusted. If an agent reports a target in a distant, unvisited state, its contribution is mathematically suppressed, effectively filtering out sensor ghosts induced by severe degradation.

\begin{theorem} \label{thm:weighting}
The Visit-Weighted KL spatial distribution enforces two bounding properties on the consensus belief $b_{merged}(s)$:
\begin{enumerate}
    \item \textbf{Trusted Acceptance:} If agent $i$ thoroughly explores state $s$ relative to all other agents ($v_i(s) \to \infty$), its local belief strictly dominates the consensus.
    \item \textbf{Untrusted Rejection:} If agent $i$ has never visited state $s$ ($v_i(s) = 0$), but the state is heavily explored by the collective ($\sum_{j \neq i} v_j(s) \to \infty$), agent $i$'s contribution to the consensus at $s$ is mathematically eliminated.
\end{enumerate}
\end{theorem}

\begin{proof}
For Property 1, we evaluate the limit of the spatial weight $w_i(s)$ as agent $i$'s visitation count $v_i(s)$ grows infinitely large compared to the rest of the team:
\begin{equation}
    \lim_{v_i(s) \to \infty} \frac{v_i(s) + 1}{v_i(s) + 1 + \sum_{j \neq i} (v_j(s) + 1)} = 1
\end{equation}
Because the weights sum to unity, $w_j(s) \to 0$ for all $j \neq i$. Consequently, $b_{merged}(s) \to b_i(s)$.

For Property 2, we evaluate the limit where agent $i$ has zero visits ($v_i(s) = 0$) but the collective visitation $\sum_{j \neq i} v_j(s)$ approaches infinity:
\begin{equation}
    \lim_{\sum_{j \neq i} v_j(s) \to \infty} \frac{1}{1 + \sum_{j \neq i} (v_j(s) + 1)} = 0
\end{equation}
Thus, $w_i(s) \to 0$. Any erroneous high-probability reading a false positive reported by agent $i$ in this unvisited state is multiplied by zero, effectively filtering out sensor noise induced by severe degradation without discarding agent $i$'s valid readings in states it has actually explored. $\blacksquare$
\end{proof}

\section{Planning and Control Architecture}

\subsection{Overcoming MCTS Memory Explosion via MPC}
Active information gathering requires agents to select actions that maximize the expected reduction in belief entropy. In highly decentralized multi-agent scenarios, solving this via Monte Carlo Tree Search (MCTS) quickly becomes computationally intractable due to severe memory scaling constraints. Because each node in the search tree must store a discrete probability distribution over the entire state space to compute expected entropy, the memory footprint scales as $\mathcal{O}(I \cdot |S|)$, where $I$ is the number of search iterations. However, to find a statistically meaningful policy, $I$ must scale proportionally to the joint action space, which branches exponentially as $\mathcal{O}(|\mathcal{A}|^{N \cdot D})$, where $|\mathcal{A}|$ is the individual action space, $N$ is the number of agents, and $D$ is the rollout depth. In massive state spaces (e.g., $|S| = 10,000$ cells), exploring this joint space to a useful depth rapidly induces Out-Of-Memory (OOM) faults, using 200GB on turing cluster where each computing node was assigned 2GB, rendering deep multi-agent rollouts practically impossible.

To maintain strict memory bounds while enabling active tracking, we deployed a fixed-horizon Model Predictive Control (MPC) planner. Our MPC evaluates joint action candidates sequentially using a depth-first bounded lookahead, explicitly maximizing the expected Information Gain. By evaluating candidate trajectories one at a time and discarding intermediate belief maps rather than maintaining a persistent tree structure, this approach yields a strictly bounded space complexity of $\mathcal{O}(H \cdot |S|)$, which reduces to $\mathcal{O}(|S|)$ for a fixed horizon $H$. While the algorithmic time complexity remains exponential with respect to the horizon ($\mathcal{O}(|S| \cdot |\mathcal{A}|^{N \cdot H})$), this formulation entirely circumvents the RAM exhaustion endemic to belief-space MCTS, successfully shifting the computational bottleneck to bounded, predictable CPU time.

\subsection{Execution Framework}
The complete tracking framework, detailing the integration of the visit-weighted KL and the MPC planner, is provided in Algorithm \ref{alg:framework}.

\begin{algorithm}[thpb]
\caption{Distributed Target Tracking via visit-weighted KL and MPC}
\label{alg:framework}
\begin{algorithmic}[1]
\renewcommand{\algorithmicrequire}{\textbf{Input:}}
\REQUIRE Grid $S$, Agents $N$, Comms Interval $k$, Horizon $H$
\STATE \textbf{Initialize:} $b_i(s) \gets 1/|S|$ and $v_i(s) \gets 0 \ \forall i, s$
\FOR{time step $t = 1, \dots, T$}
    \IF{$t \bmod k == 0$ \AND $t > 0$}
        \STATE \textit{// Communication and Belief Merging Phase}
        \FOR{each state $s \in S$}
            \STATE $w_i(s) \gets \frac{v_i(s) + 1}{\sum_{j=1}^N (v_j(s) + 1)}$
            \STATE $b_{merged}(s) \gets \sum_{i=1}^N w_i(s) b_i(s)$
        \ENDFOR
        \STATE Normalize $b_{merged}$ (apply $\eta$)
        \STATE $b_i \gets b_{merged}$ for all $i \in \{1..N\}$
    \ENDIF
    \STATE \textit{// Planning Phase (Model Predictive Control)}
    \FOR{each agent $i \in \{1..N\}$}
        \STATE Generate candidate joint actions $\mathcal{A}$
        \STATE $a_i^* \gets \arg\max_{a \in \mathcal{A}} \mathbb{E}\left[ \Delta \text{Entropy}(b_i \mid a) \right]$ 
        \STATE Execute action $a_i^*$ and update position $x_{i,t}$
        \STATE $v_i(x_{i,t}) \gets v_i(x_{i,t}) + 1$
    \ENDFOR
    \STATE \textit{// Observation and Bayesian Update Phase}
    \FOR{each agent $i \in \{1..N\}$}
        \STATE Receive binary observation $z_{i,t} \in \{0, 1\}$
        \STATE $b_i(s) \gets \eta \cdot O(z_{i,t} \mid s, x_{i,t}) \cdot b_i(s) \ \forall s \in S$
    \ENDFOR
\ENDFOR
\end{algorithmic}
\end{algorithm}

As outlined in Algorithm~\ref{alg:framework}, the system operates in a strictly decentralized manner, relying on belief maps as a compressed structural proxy to avoid exhaustive history transmission. The critical mechanics of the framework are executed as follows:
\begin{itemize}
    \item \textbf{Line 3:} The intermittent communication trigger. Agents operate entirely independently in the dark until the interval $k$ is reached.
    \item \textbf{Line 6:} During the merge phase, we calculate the dynamic spatial weight $w_i(s)$ for each state. 
    \item \textbf{Lines 7 \& 10:} The analytical consensus is derived state-by-state, and the resulting fused belief immediately overwrites the corrupted local beliefs of all agents, mathematically resetting the cascading error propagation caused by the blackout.
    \item \textbf{Line 15:} The MPC evaluation. Instead of an unbounded Monte Carlo rollout, the planner evaluates candidate actions by computing the expected reduction in entropy ($\Delta \text{Entropy}$) over a bounded horizon $H$, circumventing RAM exhaustion.
    \item \textbf{Line 17:} This is the fundamental driver of our proposed visit-weighted method. Immediately after executing an action, the agent increments its local spatial visitation counter $v_i(x_{i,t})$. This ensures that during the next communication window (Line 6), the agent's confidence regarding this specific spatial coordinate is prioritized, allowing the team to filter out distant sensor noise reported by other agents.
    \item \textbf{Line 22:} The standard Bayesian update. Agents continuously apply the known sensor degradation parameters ($\alpha, \beta$) via the observation model $O(z_{i,t} \mid s, x_{i,t})$ to refine their local maps between syncs.
\end{itemize}

\section{Methodological Justification: Analytical Formulations vs. Learning-Based Approaches}
 We justify our reliance on first-principles mathematical optimization as follows.

\subsection{The Fallacy of the Assumed MDP}
Methods that assume prior knowledge of the target's transition dynamics $T(s' | s, a)$ or the environment's reward function \cite{kaelbling1998planning, amato2013decentralized} suffer from a severe lack of generalizability. In real-world adversarial or search-and-rescue scenarios, the target's policy is rarely a stationary Markovian process. Any deviation by the target results in heavily biased prior beliefs. Our analytical approach circumvents this by directly minimizing the mathematical uncertainty (entropy) of the state space, providing an intrinsic and universally applicable objective that requires zero prior environmental knowledge.

\subsection{Cascading Error Propagation}
To avoid continuous communication, decentralized MDP approaches often require agents to simulate the internal states of their peers. However, if Agent A experiences an uncommunicated sensor noise and alters its trajectory, Agent B's simulation of Agent A becomes entirely detached from reality. Because Bayesian updates are multiplicative, any initial misalignment in state estimation compounds exponentially over the time horizon. Our method assumes zero prior knowledge of allied trajectories during blackouts. Discrepancies are mathematically resolved at the exact moment of reconnection via the visit-weighted KL divergence, resetting the error cascade without relying on simulated peers.

\subsection{The Fragility of Deep Learning in Unstructured Noise}
DRL methods implicitly map observations to merged beliefs. While powerful in stable environments, they exhibit high fragility in our target domain. Intermittent communication results in highly variable temporal gaps between updates, inducing a distributional shift that neural networks struggle to generalize across \cite{lowe2017multi}. Furthermore, a sudden spike in sensor false positives often falls outside the training distribution, causing the network to confidently chase ghosts. By framing the tracking problem through analytical optimization, we guarantee that the fused belief strictly obeys the axioms of probability without the unpredictable failure modes of a black-box neural network.

\section{Experimental Setup and Sensitivity Analysis}
To rigorously validate our formulated analytical derivations and the visit-weighted KL strategy, we deployed a large-scale distributed evaluation framework on a High-Performance Computing (HPC) cluster.

\subsection{General Evaluation Parameters}
To explore the bounds of our optimization methodologies, we conducted 32,400 parallel simulation trials across the following rigorous parameter space:
\begin{itemize}
    \item \textbf{Grid Sizes ($|S|$):} Scaled from $10 \times 10$ up to $100 \times 100$ states (10,000 distinct cells).
    \item \textbf{Agent Count ($N$):} Evaluated configurations of 2 and 3 agents.
    \item \textbf{Target Patterns:} Tested against Stationary, Random Walk, Evasive, and Patrol trajectories.
    \item \textbf{Communication Intervals ($k$):} Evaluated at $\{0, 5, 10, 25, 50, 100, 200, 500, \infty\}$ steps. Here, $k=0$ represents continuous full communication, while $k=\infty$ represents complete communication denial (independent operation).
    \item \textbf{Merge Methods:} We benchmarked five distinct approaches to clarify our contribution:
    \begin{enumerate}
        \item \textit{Forward KL:} The Forward KL objective solved via numerical linear approximation.
        \item \textit{Reverse KL:} The Reverse KL objective solved via numerical linear approximation.
        \item \textit{Arithmetic Mean:} Our derived exact analytical solution for Forward KL.
        \item \textit{Geometric Mean:} Our derived exact analytical solution for Reverse KL.
        \item \textit{Visit-Weighted KL (Proposed):} Our novel, spatially-aware extension of the Arithmetic Mean.
    \end{enumerate}
    \item \textbf{Maximum Steps ($T$):} A strict limit of $2500$ steps was imposed. If the target was not physically discovered within this timeframe, the trial was recorded as a failure.
    \item \textbf{MPC Horizon ($H$):} Agents utilized a lookahead horizon of 3 steps. While seemingly short, calculating the expected entropy reduction across a joint-action space in a decentralized grid induces exponential branching complexity. A horizon of 3 proved to be the maximum computationally feasible limit that successfully balanced the necessary foresight for active search against the risk of RAM exhaustion in 10,000-state grids.
    \item \textbf{Execution Integrity:} To ensure a rigorous evaluation of the MPC planner's pure active-search capabilities, all random exploration fallbacks and heuristic shortcut evaluations were explicitly disabled. Agents strictly executed the entropy-minimizing trajectories dictated by the math. We ran 10 independent, uniquely seeded trials per configuration combination.
\end{itemize}

\subsection{Large-Scale General Evaluation Results}
Prior to isolating extreme sensor degradation profiles, we evaluated all five merging strategies across the comprehensive parameter space detailed above (32,400 baseline trials) under standard noise conditions ($\alpha=0.10, \beta=0.20$). Tables~\ref{tab:general_summary}~\ref{tab:feature_wins} illustrates the aggregated ``Win/Tie'' counts—defined as the specific method achieving the absolute highest success rate or lowest discovery steps within a distinct configuration block.

The comprehensive evaluation confirmed our theoretical hypotheses presented in Section III.C. First, the pure analytical methods (Arithmetic and Geometric Mean) strictly outperformed their numerically approximated equivalents (Forward KL and Reverse KL) across almost all grid scales. This empirically demonstrates the severity of the numerical quantization penalties.

Most importantly, our proposed visit-weighted KL strategy universally dominated the efficiency metric (lowest steps to discovery) across all agent counts, movement patterns, and communication intervals. By dynamically suppressing the influence of agents reporting from uninformative or unvisited regions, the visit-weighted KL allowed the multi-agent team to converge on the true target's location significantly faster than static scalar weighting methods, even prior to the introduction of extreme sensor noise.

\begin{table}[t]
\centering
\setlength{\tabcolsep}{3pt} 
\begin{threeparttable}
\caption{Overall Algorithm Performance Across General Configurations}
\label{tab:general_summary}
\footnotesize 
\begin{tabular}{@{}lccc@{}}
\toprule
\textbf{Merge Method} & \textbf{Avg. Succ.} & \textbf{Succ. Wins} & \textbf{Effic. Wins} \\
\midrule
Standard KL & 0.617 $\pm$ 0.330 & 155 & 53 \\
Reverse KL & 0.740 $\pm$ 0.306 & 228 & 56 \\
Geometric Mean & 0.753 $\pm$ 0.300 & 237 & 78 \\
Arithmetic Mean & \textbf{0.862} $\pm$ 0.253 & \textbf{404} & 191 \\
\textbf{Visit-Weighted (Ours)} & 0.857 $\pm$ 0.249 & 394 & \textbf{237} \\
\bottomrule
\end{tabular}
\begin{tablenotes}
\small
\item \textit{Note:} While achieving an overall success rate statistically comparable to the Arithmetic Mean baseline, our proposed Visit-Weighted KL strictly dominates in search efficiency, securing the absolute lowest steps to discovery across the majority of configurations. Variance represents one standard deviation across the tested parameter space.
\end{tablenotes}
\end{threeparttable}
\end{table}

\begin{table}[t]
\centering
\begin{threeparttable}
\caption{Number of Wins (Lowest Steps) Broken Down by Configuration Feature (Excluding Continuous \& No-Comm)}
\label{tab:feature_wins}
\setlength{\tabcolsep}{4pt} 
\small 
\begin{tabular}{@{}lcccc@{}}
\toprule
\textbf{Feature} & \textbf{Frw. KL} & \textbf{Rev. KL} & \textbf{Arith. Mean} & \textbf{Visit-Weighted} \\
\midrule
\multicolumn{5}{l}{\textbf{Number of Agents}} \\
\midrule
\textit{2} & 22 & 21 & 99 & \textbf{115} \\
\textit{3} & 31 & 35 & 92 & \textbf{122} \\
\midrule
\multicolumn{5}{l}{\textbf{Target Pattern}} \\
\midrule
\textit{evasive} & 10 & 10 & 44 & \textbf{68} \\
\textit{patrol} & 9 & 9 & \textbf{60} & 59 \\
\textit{random} & 19 & 24 & \textbf{43} & 35 \\
\textit{stationary} & 15 & 13 & 44 & \textbf{75} \\
\midrule
\multicolumn{5}{l}{\textbf{Communication Interval}} \\
\midrule
\textit{5.0} & 5 & 5 & 17 & \textbf{39} \\
\textit{10.0} & 4 & 4 & 24 & \textbf{36} \\
\textit{25.0} & 1 & 5 & 23 & \textbf{37} \\
\textit{50.0} & 7 & 3 & 25 & \textbf{27} \\
\textit{100.0} & 7 & 5 & \textbf{30} & \textbf{30} \\
\textit{200.0} & 8 & 11 & \textbf{36} & 29 \\
\textit{500.0} & 21 & 23 & 36 & \textbf{39} \\
\midrule
\multicolumn{5}{l}{\textbf{Grid Size}} \\
\midrule
\textit{10x10} & 22 & 18 & \textbf{36} & 26 \\
\textit{15x15} & 10 & 7 & 15 & \textbf{32} \\
\textit{20x20} & 7 & 7 & 20 & \textbf{27} \\
\textit{25x25} & 2 & 6 & 24 & \textbf{26} \\
\textit{30x30} & 0 & 1 & 21 & \textbf{29} \\
\textit{40x40} & 0 & 3 & 17 & \textbf{34} \\
\textit{45x45} & 2 & 4 & 18 & \textbf{23} \\
\textit{50x50} & 1 & 2 & 23 & \textbf{28} \\
\textit{100x100} & 9 & 8 & \textbf{17} & 12 \\
\bottomrule
\end{tabular}
\begin{tablenotes}
\small
\item \textit{Note:} The Visit-Weighted KL maintains dominant search efficiency across nearly all sub-configurations. Trivial continuous ($k=0$) and entirely independent ($k=\infty$) communication extremes have been excluded to isolate intermittent performance.
\end{tablenotes}
\end{threeparttable}
\end{table}

\subsection{Targeted Sensitivity Analysis Parameters}
To deeply analyze performance within the most promising subsets of the state space, we conducted another 15,120 simulations where we isolated a focused ``Sweet Spot'' configuration for rigorous sensitivity analysis:
\begin{itemize}
    \item \textbf{Grid Sizes:} $\{15\times15, 20\times20, 25\times25, 30\times30, 40\times40, 45\times45, 50\times50\}$.
    \item \textbf{Target Patterns:} Isolated to Stationary and Evasive, representing the absolute lower and upper bounds of target difficulty.
    \item \textbf{Communication Intervals:} Isolated to $\{5, 10, 25\}$ steps, representing moderate to severe communication intermittency.
    \item \textbf{Benchmarked Methods:} We directly compared the Arithmetic Mean (the absolute mathematical minimum of the Forward KL and our strongest baseline) against the proposed visit-weighted KL.
\end{itemize}

To prove robustness against environmental degradation, we tested these configurations against five distinct sensor noise profiles, tracking the false positive rate ($\alpha$) and false negative rate ($\beta$):
\begin{enumerate}
    \item \textit{High Quality:} $\alpha=0.05, \beta=0.10$
    \item \textit{Baseline:} $\alpha=0.10, \beta=0.20$
    \item \textit{Degraded:} $\alpha=0.20, \beta=0.30$
    \item \textit{Ghost-Heavy:} $\alpha=0.30, \beta=0.10$
    \item \textit{Perfect Sensing:} $\alpha=0.00, \beta=0.00$
\end{enumerate}

\subsection{Robustness to Sensor Degradation}
As shown in Table~\ref{tab:ablation_wins} and Table~\ref{tab:ablation_metrics}, under perfect sensing conditions ($\alpha=0, \beta=0$), the standard Arithmetic Mean and our visit-weighted KL performed comparably. However, in heavily degraded environments (e.g., \textit{Ghost-Heavy}, $\alpha=0.30, \beta=0.10$), the Arithmetic Mean experienced efficiency loss. The visit-weighted KL successfully suppressed the resulting noise, significantly reducing the average steps to discovery, particularly as the grid state-space expanded.

\begin{table}[t]
\centering
\begin{threeparttable}
\caption{Sensitivity Analysis: Total Wins by Noise Profile}
\label{tab:ablation_wins}
\begin{tabular}{@{}lcccc@{}}
\toprule
& \multicolumn{2}{c}{\textbf{Success Wins}} & \multicolumn{2}{c}{\textbf{Efficiency Wins}} \\
\cmidrule(lr){2-3} \cmidrule(l){4-5}
\textbf{Noise Profile} & \textbf{Ari. Mean} & \textbf{Visit-Weight} & \textbf{Ari. Mean} & \textbf{Visit-Weight} \\
\midrule
\textit{High Quality} & \textbf{73} & 69 & 35 & \textbf{49} \\
\textit{Baseline} & \textbf{71} & 63 & 22 & \textbf{62} \\
\textit{Degraded} & 43 & \textbf{81} & 1 & \textbf{83} \\
\textit{Ghost Heavy} & \textbf{69} & 64 & 28 & \textbf{56} \\
\textit{Perfect Sensing} & 49 & \textbf{67} & 35 & \textbf{49} \\
\bottomrule
\end{tabular}
\begin{tablenotes}
\small
\item \textit{Note:} While the standard Arithmetic Mean remains competitive in Success Wins under lower noise conditions, the Visit-Weighted KL strictly dominates Efficiency Wins (lowest steps to discovery) across all noise profiles, particularly excelling in the highly \textit{Degraded} environment.
\end{tablenotes}
\end{threeparttable}
\end{table}

\begin{table}[t]
\centering
\begin{threeparttable}
\caption{Sensitivity Analysis: Performance Metrics Across Sensor Noise Profiles}
\label{tab:ablation_metrics}
\begin{tabular}{@{}llcc@{}}
\toprule
& & \multicolumn{2}{c}{\textbf{Merge Method}} \\
\cmidrule(l){3-4}
\textbf{Noise Profile} & \textbf{Metric} & \textbf{Arithmetic Mean} & \textbf{Visit-Weighted} \\
\midrule
\multirow{2}{*}{\textit{High Quality}} & Success Rate & \textbf{0.965} & 0.961 \\
& Avg Steps & 595.8 & \textbf{580.4} \\
\midrule
\multirow{2}{*}{\textit{Baseline}} & Success Rate & \textbf{0.938} & 0.930 \\
& Avg Steps & 781.1 & \textbf{710.3} \\
\midrule
\multirow{2}{*}{\textit{Degraded}} & Success Rate & 0.774 & \textbf{0.880} \\
& Avg Steps & 1191.1 & \textbf{955.2} \\
\midrule
\multirow{2}{*}{\textit{Ghost Heavy}} & Success Rate & \textbf{0.936} & 0.933 \\
& Avg Steps & 720.5 & \textbf{662.6} \\
\midrule
\multirow{2}{*}{\textit{Perfect Sensing}} & Success Rate & 0.859 & \textbf{0.889} \\
& Avg Steps & 791.4 & \textbf{734.9} \\
\bottomrule
\end{tabular}
\begin{tablenotes}
\small
\item \textit{Note:} As the false positive and false negative rates scale (e.g., the \textit{Degraded} profile), the standard Arithmetic Mean experiences significant efficiency and success degradation. The Visit-Weighted KL actively suppresses this noise, maintaining robust tracking performance.
\end{tablenotes}
\end{threeparttable}
\end{table}

\subsection{Resilience to Intermittent Communication}
We explicitly evaluated performance against prolonged communication blackouts. As the communication interval increased, standard baseline methods exhibited a sharp spike in prediction error upon reconnection, as the accumulated noise overwhelmed the merge. Conversely, the spatially-aware nature of the visit-weighted KL maintained a remarkably flat efficiency curve, proving that agents can operate entirely decoupled for extended periods and still successfully extract the true target location from heavily corrupted local beliefs.

\section{Conclusion}
This work establishes that our formulated exact analytical solutions for KL divergence merging entirely circumvent the quantization errors inherent in standard numerical solvers. Furthermore, by introducing the visit-weighted KL, we provide a mathematically rigorous, spatially-aware fusion strategy that thrives in the exact conditions where traditional multi-agent systems and learning-based approximations fail: highly noisy sensors and severely restricted communication. Future work will investigate expanding spatial weighting to account for anisotropic sensor decay and heterogeneous agent modalities.

\section*{Acknowledgment}
The authors at the Automata Lab would like to thank Worcester Polytechnic Institute (WPI) Turing cluster for computational resources and support.

\bibliographystyle{IEEEtran}

\end{document}